\def\year{2019}\relax
\documentclass[letterpaper]{article} 
\usepackage{aaai19}  
\usepackage{times}  
\usepackage{helvet}  
\usepackage{courier}  
\usepackage{url}  
\usepackage{graphicx}  
\usepackage{tikz}
\usepackage{amsmath}
\usepackage{amsthm}
\usepackage{amssymb}
\usepackage{listings}
\usepackage{siunitx}
\usepackage{booktabs}

\mathchardef\mh="2D
\newcommand{\mr}[1]{\ensuremath{ \mathrm{#1}}}

\newtheorem{theorem}{Theorem}
	
\newtheorem{prop}        [theorem] {Proposition}	

\newcommand{\pre}{\mathsf{pre}}     
\newcommand{\eff}{\mathsf{eff}}     
\newcommand{\cond}{\mathsf{cond}}   

\begin{document}
%
\title{Solving Multiagent Planning Problems with Concurrent Conditional Effects}
\author{Daniel Furelos-Blanco\thanks{The work was conducted while at Universitat Pompeu Fabra.}\\
Department of Computing\\
Imperial College London\\
London, SW7 2AZ, United Kingdom\\
d.furelos-blanco18@imperial.ac.uk
\And
Anders Jonsson\\
Dept.~Information and Communication Technologies\\
Universitat Pompeu Fabra\\
Roc Boronat 138, 08018 Barcelona, Spain\\
anders.jonsson@upf.edu
}
\maketitle

\begin{abstract}
	In this work we present a novel approach to solving concurrent multiagent planning problems in which several agents act in parallel. Our approach relies on a compilation from concurrent multiagent planning to classical planning, allowing us to use an off-the-shelf classical planner to solve the original multiagent problem. The solution can be directly interpreted as a concurrent plan that satisfies a given set of concurrency constraints, while avoiding the exponential blowup associated with concurrent actions. Our planner is the first to handle action effects that are conditional on what other agents are doing. Theoretically, we show that the compilation is sound and complete. Empirically, we show that our compilation can solve challenging multiagent planning problems that require concurrent actions.
\end{abstract}

\section{Introduction}
Concurrent multiagent planning is a branch of multiagent planning in which several agents {\em collaborate} to solve a given problem. Collaboration takes the form of {\em concurrent} or {\em joint} actions that are executed together by multiple agents. Concurrent multiagent planning is challenging for several reasons: the number of concurrent actions is worst-case exponential in the number of agents, and restrictions are needed to ensure that concurrent actions are well-formed. Usually, these restrictions take the form of {\em concurrency constraints}~\cite{BoutilierB01,Cr13b}, which model both the case for which two actions {\em must} occur in parallel, and for which they {\em cannot} occur in parallel.

In spite of recent progress in multiagent planning, there are relatively few multiagent planners that can reliably handle concurrency. CMAP~\cite{Borrajo13}, MAPlan~\cite{StolbaFK16} and MH-FMAP~\cite{TorrenoOS14} can all produce concurrent plans, but are not designed to handle more complex concurrency constraints. \citeauthor{CrosbyJR14}~(\citeyear{CrosbyJR14}) associate concurrency constraints with the {\em objects} of a multiagent planning problem and transform the problem into a sequential, single-agent problem that can be solved using a classical planner. \citeauthor{ShekharB18}~(\citeyear{ShekharB18}) adapt this approach using {\em collaborative actions}, i.e.~single actions that involve the minimum number of agents necessary to perform a given task. \citeauthor{BrafmanZoran14}~(\citeyear{BrafmanZoran14}) extend the distributed forward-search planner MAFS \cite{NissimB14} to support concurrency constraints while preserving privacy. \citeauthor{MaliahBS17}~(\citeyear{MaliahBS17}) propose MAFBS, which extends MAFS to use forward and backward messages.

In this paper we describe a planner that can handle arbitrary concurrency constraints, as originally proposed by~\citeauthor{BoutilierB01}~(\citeyear{BoutilierB01}) and later extended by \citeauthor{kovacs2012multi}~(\citeyear{kovacs2012multi}). Our approach is similar to previous approaches in that we transform a multiagent planning problem into a single-agent problem with much fewer actions, avoiding the exponential blowup associated with concurrent actions. The concurrency constraints of \citeauthor{BoutilierB01} are significantly more expressive than those of~\citeauthor{Cr13b}~(\citeyear{Cr13b}), enabling us to solve multiagent problems with more complex interactions (e.g.~effects that depend on the concurrent actions of other agents). We show that our planner is sound and complete, and perform experiments in several concurrent multiagent planning domains to evaluate its performance.

The remainder of this paper is structured as follows. We first introduce the planning formalisms that we need to describe our planner. Next, we describe the compilation from multiagent planning to single-agent planning. We then present the results of experiments in several domains that require concurrency. Finally, we relate our planner to existing work in the literature, and conclude with a discussion.

\section{Background}

In this section we describe the planning formalisms that we use: classical planning and concurrent multiagent planning.

\subsection{Classical Planning}

We consider the fragment of classical planning with conditional effects and negative conditions and goals. Given a fluent set $F$, a {\em literal} $l$ is a valuation of a fluent in $F$, where $l=f$ denotes that $l$ assigns true to $f\in F$, and $l=\neg f$ that $l$ assigns false to $f$. A literal set $L$ is {\em well-defined} if it does not
assign conflicting values to any fluent $f$, i.e.~does not
contain both $f$ and $\neg f$. Let $\mathcal{L}(F)$ be the set of well-defined literal sets on $F$, i.e.~the set of all partial assignments of values to fluents. Given a literal set $L\in\mathcal{L}(F)$, let $\neg L=\{\neg l:l\in L\}$ be the {\em complement} of $L$. We also define the {\em projection} $L_{|X}$ of a literal set $L$ onto a subset of fluents $X\subseteq F$.

A {\em state} $s\in\mathcal{L}(F)$ is a well-defined literal set such that $|s|=|F|$, i.e.~a total assignment of values to fluents. Explicitly including negative literals $\neg f$ in states simplifies subsequent definitions, but we often abuse notation by defining a state $s$ only in terms of the fluents that are true in $s$, as is common in classical planning.

A classical planning problem is a tuple $\Pi=\left\langle F,A,I,G \right\rangle$, where $F$ is a set of fluents, $A$ a set of actions, $I\in\mathcal{L}(F)$ an initial state, and $G\in\mathcal{L}(F)$ a goal condition (usually satisfied by multiple states). Each action $a \in A$ has a precondition $\pre(a)\in\mathcal{L}(F)$ and a set of conditional effects $\cond(a)$. Each conditional effect $C\rhd E\in\cond(a)$ has two literal sets $C\in\mathcal{L}(F)$ (the condition) and $E\in\mathcal{L}(F)$ (the effect). 

An action $a\in A$ is applicable in state $s$ if and only if $\pre(a)\subseteq s$, and the resulting (triggered) {\em effect} is given by
\[
\eff(s,a)=\bigcup_{C\rhd E\in\cond(a),C\subseteq s} E,
\]
i.e.~effects whose conditions hold in $s$. We assume that $\eff(s,a)$ is a well-defined literal set in $\mathcal{L}(F)$ for each state-action pair $(s,a)$. The result of applying $a$ in $s$ is a new state $\theta(s,a)=(s\setminus \neg\eff(s,a))\cup\eff(s,a)$. It is straightforward to show that if $s$ and $\eff(s,a)$ are in $\mathcal{L}(F)$, then so is $\theta(s,a)$.

A {\em plan} for planning problem $\Pi$ is an action sequence $\pi=\langle a_1, \ldots, a_n\rangle$ that induces a state sequence $\langle s_0, s_1, \ldots, s_n\rangle$ such that $s_0=I$ and, for each $i$ such that $1\leq i\leq n$, action $a_i$ is applicable in $s_{i-1}$ and generates the successor state $s_i=\theta(s_{i-1},a_i)$. Plan $\pi$ {\em solves} $\Pi$ if and only if $G\subseteq s_n$, i.e.~if the goal condition holds after applying $\pi$ in $I$.

\subsection{Concurrent Multiagent Planning}

The standard definition of multiagent planning problems (MAPs) is due to~\citeauthor{BrafmanD08}~(\citeyear{BrafmanD08}). Formally, a MAP is a tuple $\Pi = \langle N,F,\lbrace A^i\rbrace_{i\in N},I,G \rangle$, where $N = \left\lbrace1,\ldots,n\right\rbrace$ is a set of agents and $A^1,\ldots,A^n$ are disjoint sets of atomic actions of each agent. The fluent set $F$, initial state $I$ and goal condition $G$ are defined as for classical (single-agent) planning. The definition and semantics of a plan $\pi$ are also identical to those for classical planning, except that $\pi$ is a sequence of {\em joint actions}, which we proceed to define.

Let $A=A^1\cup\cdots\cup A^n$ be the set of all atomic actions. A joint action $a=\{a^1,\ldots, a^k\}\subseteq A$ is a subset of atomic actions such that $|A^i\cap a|\leq 1$ for each $i\in N$, i.e.~each agent contributes at most one action to $a$, implying $k\leq n$. The precondition and effect of $a$ are defined as the union of the preconditions and effects of the constituent atomic actions:
\[
\pre(a)=\bigcup_{j=1}^k \pre(a^j), \;\;\;\; \eff(s,a)=\bigcup_{j=1}^k \eff(s,a^j).
\]
A joint action $a$ is {\em well-defined} if $\pre(a)$ and $\eff(s,a)$ are well-defined literal sets in $\mathcal{L}(F)$ for each state $s$. 

In general, nothing prevents two atomic actions $a^1$ and $a^2$ of different agents from having conflicting preconditions or effects. Hence any joint action that includes both $a^1$ and $a^2$ is not well-defined. Moreover, some atomic actions may only be applicable together. For example, in the \textsc{BoxPushing} domain~\cite{BrafmanZoran14}, some boxes are too heavy to push for a single agent, and a joint action is only applicable if enough agents push a box concurrently.

To ensure that joint actions are applicable and well-defined, researchers usually impose {\em concurrency constraints} on joint actions, which can be either positive or negative:
\begin{itemize}
	\item A \textit{positive concurrency constraint} states that a subset of atomic actions \textit{must} be performed concurrently.
	\item A \textit{negative concurrency constraint} states that a subset of atomic actions \textit{cannot} be performed concurrently.
\end{itemize}

In PDDL 2.1~\cite{Fox:PDDL21:JAIR2003}, two actions $a^1$ and $a^2$ cannot be applied concurrently if $a^1$ has an effect on a fluent $f$ and $a^2$ has a precondition or effect on $f$. This concurrency constraint requires no prior knowledge apart from the action definitions, but can only model negative concurrency.

\citeauthor{Cr13b}~(\citeyear{Cr13b}) defines concurrency constraints in the form of {\em object affordances}, i.e.~integer intervals that determine how many agents can interact with an object concurrently. An object with affordance $[1,1]$ can only be manipulated by one agent, while $[2,10]$ requires manipulation by at least two and at most ten agents. This approach assumes that joint actions are well-defined whenever object affordances are satisfied. \citeauthor{CrosbyJR14}~(\citeyear{CrosbyJR14}) extend the approach to affordances on {\em object sets}.

\citeauthor{BoutilierB01}~(\citeyear{BoutilierB01}) proposed an alternative definition of concurrency constraints, later extended by \citeauthor{kovacs2012multi}~(\citeyear{kovacs2012multi}). The idea is to extend the preconditions of actions with {\em other actions} in addition to fluents. If an atomic action $a^1$ has precondition $a^2$, then $a^2$ must be applied concurrently with $a^1$, while a precondition $\neg a^2$ implies that $a^2$ cannot be applied concurrently with $a^1$. A joint action is only applicable if the concurrency constraints (i.e.~preconditions) of {\em all} constituent atomic actions hold, and applicable joint actions are assumed to be well-defined.

As a side effect of the latter approach, we can also add concurrency constraints to the conditional effects of atomic actions. We illustrate this idea using the \textsc{TableMover} domain~\cite{BoutilierB01}, in which two agents move blocks between rooms using two alternative strategies:
\begin{enumerate}
	\item Pick up blocks and carry them using their arms.
	\item Put blocks on a table, carry the table together to another room, and tip the table to make the blocks fall down.
\end{enumerate}

\begin{figure}
	\begin{footnotesize}
		\begin{lstlisting}[escapechar=ä]
(:action lift-side
 ä\textbf{:agent ?a - agent}ä
 :parameters (?s - side)
 :precondition 
   (and (at-side ?a ?s)
          (down ?s) (handempty ?a)
          ä\textbf{(forall} \textbf{(?a2 - agent ?s2 - side)}ä
            ä\textbf{(not (lower-side ?a2 ?s2)))}ä)
 :effect
   (and (not (down ?s)) (lifting ?a ?s)
          (up ?s) (not (handempty ?a ?s))
          (forall 
            (?b - block ?r - room ?s2 - side)
            (when 
              (and (inroom Table ?r)
                     (on-table ?b) (down ?s2)
                     ä\textbf{(forall (?a2 - agent)}ä
                       ä\textbf{(not (lift-side ?a2 ?s2)))}ä)
              (and (on-floor ?b) (inroom ?b ?r)
                     (not (on-table ?b)))))))
		\end{lstlisting}
	\end{footnotesize}
	\caption{Definition of the \textsc{TableMover} action $\mathsf{lift\mh side}$ using the notation of Kovacs (concurrency constraints in bold).}
	\label{fig:tablemover_liftside}
\end{figure}

Figure~\ref{fig:tablemover_liftside} shows the definition of the $\mathsf{lift\mh side}$ action in the notation of~\citeauthor{kovacs2012multi}~(\citeyear{kovacs2012multi}), which is used by agent $\mathsf{?a}$ to lift side $\mathsf{?s}$ of the table. The precondition is that the side must be down (i.e.~on the floor) and the agent cannot be holding anything. Moreover, the precondition also states that no other agent $\mathsf{?a2}$ can lower side $\mathsf{?s2}$ at the same time. When the action is applied, $\mathsf{?s}$ is no longer down but up, and $\mathsf{?a}$ is busy lifting $\mathsf{?s}$. The action also has a conditional effect (represented by the $\mathsf{when}$ clause): if some side $\mathsf{?s2}$ is not lifted by any agent $\mathsf{?a2}$, then all blocks on the table fall to the floor. This conditional effect is what makes it possible to tip the table in order to implement the second strategy above.

Note that the action $\mathsf{lift\mh side}$ is defined using $\mathsf{forall}$ quantifiers. In practice, such quantifiers are compiled away, such that the resulting actions have quantifier-free preconditions and conditional effects, as in our definition of actions.

Below we extend the notation for classical planning to incorporate the concurrency constraints of \citeauthor{BoutilierB01}~(\citeyear{BoutilierB01}). The idea is to view $A$, the full set of atomic actions, as a set of fluents that can be true or false. We can now use a set of literals on $A$ to model a joint action $a=\{a^1,\ldots, a^k\}$: the fluents in $A$ corresponding to actions $a^1,\ldots, a^k$ are true, while all other fluents in $A$ are false. Let $L(a)$ denote the literal set on $A$ that encodes $a$.

The next step is to define an extended fluent set $F\cup A$, as well as a set $\mathcal{L}(F\cup A)$ of well-defined literal sets on $F\cup A$. We can now encode a state $s$ and a joint action $a$ as an extended state $s\cup L(a)$, i.e.~a literal set in $\mathcal{L}(F\cup A)$.

To include concurrency constraints in the precondition and conditional effects of an action $a^j$, we simply define the precondition $\pre(a^j)\in\mathcal{L}(F\cup A)$ and condition $C\in\mathcal{L}(F\cup A)$ of each conditional effect $C\rhd E\in\cond(a^j)$ as well-defined literal sets on extended fluents. Each effect $E\in\mathcal{L}(F)$ is defined exclusively on fluents as before.

We can now define the semantics of a joint action $a=\{a^1,\ldots, a^k\}$. Concretely, $a$ satisfies the concurrency constraints if and only if the projected precondition $\pre(a^j)_{|A}$ holds in $L(a)$ for each $j$, $1\leq j\leq k$. If $a$ is applicable, its precondition and effect in a state $s$ are the union of the preconditions and effects of the constituent atomic actions:
\begin{align*}
\pre(a)&=\bigcup_{j=1}^k\pre(a^j)_{|F},\\
\eff(s,a)&=\bigcup_{j=1}^k\eff(s\cup L(a), a^j), \;\; \forall s.
\end{align*}
Note that the effects of atomic actions are conditional on the extended state $s\cup L(a)$.
As before, we assume that $\eff(s,a)$ is a well-defined literal set in $\mathcal{L}(F)$ for each pair of a state $s$ and an applicable joint action $a$.

We use a small example to illustrate the notation. Consider a MAP with two agents and action sets $A^1=\{a^1,a^2\}$ and $A^2=\{a^3,a^4\}$. The full set of actions is $A=A^1\cup A^2=\{a^1,a^2,a^3,a^4\}$. Assume that $a^1$ and $a^3$ are defined as
\begin{align*}
\pre(a^1) &= \{\neg a^4\}, & \pre(a^3) &= \emptyset,\\
\cond(a^1) &= \{\{\neg a^3\}\rhd\{f\}\}, & \cond(a^3) &= \{\emptyset\rhd\{g\}\}.
\end{align*}
The joint action $a=\{a^1,a^4\}$ is not applicable since the precondition $\neg a^4$ of $a^1$ does not hold in the extended state $L(a)=\{a^1,\neg a^2,\neg a^3,a^4\}$. The joint action $a'=\{a^1,a^3\}$ is applicable and results in the effect $\eff(s,a')=\{g\}$ in any state $s$. The joint action $a''=\{a^1\}$ is also applicable and results in the effect $\eff(s,a'')=\{f\}$ in any state $s$, since the condition $\neg a^3$ in the conditional effect $\{\neg a^3\}\rhd\{f\}$ of $a^1$ holds in the extended state $L(a'')=\{a^1,\neg a^2,\neg a^3,\neg a^4\}$.

\section{Compilations for MAPs}
\label{sec:compilation}

In this section we describe an approach to solving a MAP $\Pi= \langle N,F,\lbrace A^i\rbrace_{i\in N},I,G \rangle$. The idea is to model each joint action $a=\{a^1,\ldots,a^k\}$ using multiple atomic actions: one set of actions for {\em selecting} $a^1,\ldots,a^k$, one set of actions for {\em applying} $a^1,\ldots,a^k$, and one set of actions for {\em resetting} $a^1,\ldots,a^k$. The result is a classical planning problem $\Pi'=\langle F',A',I',G'\rangle$ such that the size of the action set $A'$ is {\em linear} in $|A|$, the number of atomic actions of agents.

Simulating a joint action $a$ using a sequence of atomic actions $\langle a^1,\ldots,a^k\rangle$ is problematic for the following reason: when applying an atomic action $a^i$, we may not yet know which atomic actions will be applied by other agents. Since those other actions may be part of the precondition and conditional effects of $a^i$, it becomes difficult to ensure that the concurrency constraints of $a^i$ are correctly enforced.

Our approach is to divide the simulation of a joint action $a$ into three phases: selection, application, and reset. In the selection phase, we use an auxiliary fluent $\mathsf{active\mh}a^i$ to model that the atomic action $a^i$ has been selected. In the application phase, since the selection of atomic actions is known, we can substitute each action $a^i$ in preconditions and conditional effects with the auxiliary fluent $\mathsf{active\mh}a^i$. In the reset phase, various auxiliary fluents are reset. Note that each agent can apply at most one atomic action per time step, and agents collaborate to form joint actions whose constituent atomic actions are compatible and/or inapplicable on their own. 

\subsection{Fluents}
We describe the fluents in PDDL format, i.e.~each fluent is instantiated by assigning objects to predicates.

The set of fluents $F'\supseteq F$ includes all original fluents in $F$, plus the following auxiliary fluents:
\begin{itemize}
	\item Fluents $\mathsf{free}$, $\mathsf{select}$, $\mathsf{apply}$ and $\mathsf{reset}$ modeling the phase.
	\item For each agent $i$, fluents $\mathsf{free\mh agent}(i)$, $\mathsf{busy\mh agent}(i)$ and $\mathsf{done\mh agent}(i)$ that model the agent state: free to select an action, selected an action, and applied the action.
	\item For each action $a^i\in A^i$ in the action set of agent $i$, a fluent $\mathsf{active\mh}a^i$ which models that $a^i$ has been selected. We use $F_{act}$ to denote the subset of fluents of this type.
\end{itemize}
By simple inspection, the total number of fluents in $F'$ is given by $|F'|=|F| + 4 + 3n + \sum_{i\in N}\left|A^i\right|=O(|F|+|A|)$.

 The initial state $I'$ of the compilation $\Pi'$ is given by
\[
I'=I\cup\{\mathsf{free}\}\cup\{\mathsf{free\mh agent}(i):i\in N\},
\]
i.e.~the initial state on fluents in $F$ is $I$, we are not simulating any joint action, and all agents are free to select actions. The goal condition is given by $G'=G\cup\{\mathsf{free}\}$, i.e.~the goal condition $G$ has to hold at the end of a joint action simulation.

\subsection{Actions}

For a literal set $L\in\mathcal{L}(F\cup A)$, let $L_{|A}/F_{act}$ denote the projection of $L$ onto $A$, followed by a substitution of the actions in $A$ with the corresponding fluents in $F_{act}$. Note that both $L_{|F}$ and $L_{|A}/F_{act}$ are literal sets on fluents in $F'$, i.e.~the dependence on actions in $A$ is removed.

The first four actions in the set $A'$ allow us to switch between simulation phases, and are defined as follows:
\begin{table}[h]
\noindent
\begin{tabular}{ll}
	$\mathsf{select\mh phase}$: & $\mr{pre}=\{\mathsf{free}\}$,\\[2pt]
	& $\mr{cond}=\{\emptyset\rhd\{\neg\mathsf{free},\mathsf{select}\}\}$.\\[2pt]
	$\mathsf{apply\mh phase}$: & $\mr{pre}=\{\mathsf{select}\}$,\\[2pt]
	& $\mr{cond}=\{\emptyset\rhd\{\neg\mathsf{select},\mathsf{apply}\}\}$.\\[2pt]
	$\mathsf{reset\mh phase}$: & $\mr{pre}=\{\mathsf{apply}\}$,\\[2pt]
	& $\mr{cond}=\{\emptyset\rhd\{\neg\mathsf{apply},\mathsf{reset}\}\}$.\\[2pt]
	$\mathsf{finish}$: & $\mr{pre}=\{\mathsf{reset},\mathsf{free\mh agent}(i): i\in N\}$,\\[2pt]
	& $\mr{cond}=\{\emptyset\rhd\{\neg\mathsf{reset},\mathsf{free}\}\}$.
\end{tabular}
\end{table}

For each action $a^i\in A^i$ in the action set of agent $i$, we define three new actions in $A'$: $\mathsf{select\mh}a^i$, $\mathsf{do\mh}a^i$ and $\mathsf{end\mh}a^i$. These actions represent the three steps that an agent must perform during the simulation of a joint action.

The action $\mathsf{select\mh}a^i$ causes $i$ to select action $a^i$ during the selection phase, and is defined as follows:
\begin{align*}
\mr{pre}&=\{\mathsf{select, free\mh agent}(i)\} \cup \pre(a^i)_{|F},\\
\mr{cond}&=\{\emptyset\rhd\{\mathsf{busy\mh agent}(i)\mathsf{, \neg free\mh agent}(i)\mathsf{, active\mh}a^i\}\}.
\end{align*}
The precondition ensures that we are in the selection phase, that $i$ is free to select an action, and that the precondition of $a^i$ holds on fluents in $F$. The effect prevents $i$ from selecting another action, and marks $a^i$ as selected.

The action $\mathsf{do\mh}a^i$ applies the effect of $a^i$ in the application phase, and is defined as follows:
\begin{align*}
\mr{pre} &= \{\mathsf{apply, busy\mh agent}(i)\mathsf{, active\mh}a^i\}\cup\pre(a^i)_{|A}/F_{act},\\
\mr{cond}&=\{\emptyset\rhd\{\mathsf{done\mh agent}(i)\mathsf{, \neg busy\mh agent}(i)\}\}\\
&\,\cup\{C_{|F}\cup C_{|A}/F_{act} \rhd E : C \rhd E \in \mathsf{cond}(a^i) \}.
\end{align*}
The precondition ensures that we are in the application phase, that $a^i$ was previously selected, and that all concurrency constraints in the precondition of $a^i$ hold. The effect is to apply all conditional effects of $a^i$, where each condition $C_{|F}\cup C_{|A}/F_{act}$ is generated from $C$ by substituting each action $a^j\in A$ with $\mathsf{active\mh}a^j$. Agent $i$ is also marked as done to prevent $a^i$ from being applied a second time.

The action $\mathsf{end\mh}a^i$ resets auxiliary fluents to their original value, and is defined as follows:
\begin{align*}
\mr{pre}&=\{\mathsf{reset, done\mh agent}(i)\mathsf{, active\mh}a^i\},\\
\mr{cond}&=\{\emptyset\rhd\{\mathsf{free\mh agent}(i)\mathsf{, \neg done\mh agent}(i), \mathsf{\neg active\mh}a^i\}\}.
\end{align*}
The precondition ensures that we are in the reset phase and that $a^i$ was previously selected and applied (due to $\mathsf{done\mh agent}(i)$). The effect is to make agent $i$ free to select actions again, and to mark $a^i$ as no longer selected.

Again, by inspection we can see that the total number of actions in $A'$ is given by $|A'|=4 + 3\sum_i|A^i|=O(|A|)$.

\subsection{Properties}

\begin{figure}
	\centering
	\includegraphics[width=\columnwidth]{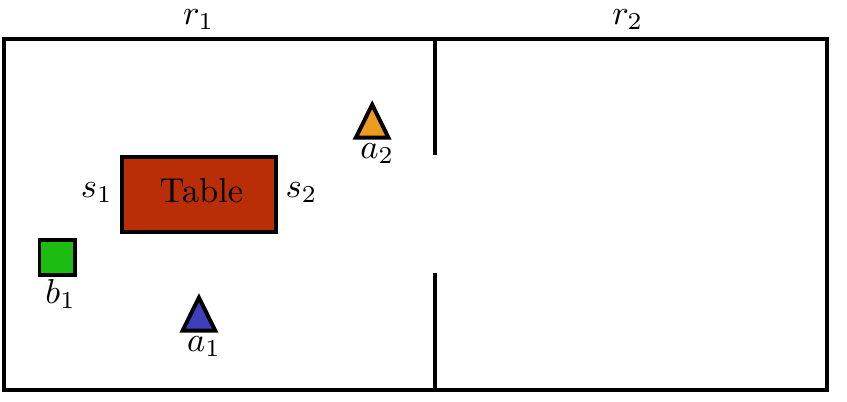}%
	\caption{Initial state of a simple \textsc{TableMover} instance.}
	\label{fig:tablemover_example_fig}
\end{figure}

Figure~\ref{fig:tablemover_example_fig} shows an example instance of \textsc{TableMover} in which the goal is for agents $a_1$ and $a_2$ to move block $b_1$ from room $r_1$ to room $r_2$. An example of concurrent plan that solves this instance is defined as follows:

\begin{lstlisting}[numbers=left, numbersep=2pt, xleftmargin=1.1em, basicstyle=\fontsize{8}{8}\ttfamily]
(to-table a1 r1 s2)(pickup-floor a2 b1 r1)
(putdown-table a2 b1 r1)
(to-table a2 r1 s1)
(lift-side a1 s2)(lift-side a2 s1)
(move-table a1 r1 r2 s2)(move-table a2 r1 r2 s1)
(lower-side a1 s2)
\end{lstlisting}
In this plan, agent $a_2$ first puts the block on the table, and then $a_1$ and $a_2$ concurrently lift each side of the table and move the table to room $r_2$. Finally, $a_1$ lowers its side of the table, causing the table to tip and the block to fall to the floor.

The following sequence of classical actions in $A'$ can be used to simulate the first joint action of the concurrent plan:
\begin{lstlisting}[numbers=left, numbersep=2pt, xleftmargin=1.1em, basicstyle=\fontsize{8}{8}\ttfamily]
(select-phase )
(select-to-table a1 r1 s2)
(select-pickup-floor a2 b1 r1)
(apply-phase )
(do-pickup-floor a2 b1 r1)
(do-to-table a1 r1 s2)
(reset-phase )
(end-to-table a1 r1 s2)
(end-pickup-floor a2 b1 r1)
(finish )
\end{lstlisting}

\noindent
We show that the compilation is both sound and complete.

\begin{theorem}[Soundness] A classical plan $\pi'$ that solves $\Pi'$ can be transformed into a concurrent plan $\pi$ that solves $\Pi$.
\end{theorem}

\begin{proof}
When fluent $\mathsf{free}$ is true, the only applicable action is $\mathsf{select\mh phase}$. The only way to make $\mathsf{free}$ true again is to cycle through the three phases and end with the $\mathsf{finish}$ action.

During the selection phase, a subset of actions $a^1,\ldots,a^k$ are selected, causing the corresponding agents to be busy. Because of the precondition $\mathsf{free\mh agent}(i)$ of the $\mathsf{finish}$ action, each selected action $a^i$ has to be applied in the application phase, and reset in the reset phase. The resulting simulated joint action is given by $a=\{a^1,\ldots,a^k\}$.

The precondition of $a$ holds since the precondition of each $a^i$ on fluents in $F$ is checked in the selection phase, during which no fluents in $F$ change values. The concurrency constraints of $a^i$ are checked in the application phase when all actions have already been selected. This also ensures that the conditional effects of $a^i$ are correctly applied. Finally, auxiliary fluents are cleaned in the reset phase. Hence the joint action $a$ satisfies all concurrency constraints and is correctly simulated by the corresponding action subsequence of $\pi'$.

Let $\pi$ be the concurrent plan composed of the sequence of joint actions simulated by the plan $\pi'$. Since $\pi'$ solves $\Pi'$, the goal condition $G$ holds at the end of $\pi'$, implying that $G$ also holds at the end of $\pi$. This implies that $\pi$ solves $\Pi$.
\end{proof}

\begin{theorem}[Completeness] A concurrent plan $\pi$ that solves $\Pi$ corresponds to a classical plan $\pi'$ that solves $\Pi'$.
\end{theorem}

\begin{proof}
Let $a=\{a^1,\ldots,a^k\}$ be a joint action of the concurrent plan $\pi$. We can use a sequence of actions in $A'$ to simulate $a$ by selecting, applying and resetting each action among $a^1,\ldots,a^k$. Since $a$ is part of $\pi$, its precondition and concurrency constraints have to hold, implying that the precondition and concurrency constraints of each atomic action hold. Hence the action sequence is applicable and results in the same effect as $a$. By concatenating such action sequences for each joint action of $\pi$, we obtain a plan $\pi'$. Since $\pi$ solves $\Pi$, the goal condition $G$ holds at the end of $\pi$, implying that $G$ holds at the end of $\pi'$. This implies that $\pi'$ solves $\Pi'$.
\end{proof}

\subsection{Extensions}
The basic compilation checks concurrency constraints in the application phase. Here we describe an extension that checks negative concurrency constraints in the selection phase, allowing a classical planner to identify inadmissible joint actions as early as possible, reducing the branching factor.

Assume that action $a^i$ has a negative concurrency constraint $\neg a^j$. As before, we can simulate this constraint using the fluent $\neg \mathsf{active\mh}a^j$. However, $a^j$ may be selected {\em after} $a^i$ in the selection phase, in which case $\neg \mathsf{active\mh}a^j$ holds when selecting $a^i$. To prevent inadmissible joint actions from being selected, we introduce additional fluents in the set $F'$:

\begin{itemize}
	\item For each action $a^i\in A^i$ in the action set of agent $i$, a fluent $\mathsf{req\mh neg\mh}a^i$ which indicates that $a^i$ cannot be selected.
\end{itemize}
We now redefine the action $\mathsf{select\mh}a^i$ as follows:
\begin{align*}
\mr{pre}&=\{\mathsf{select, free\mh agent}(i),\neg\mathsf{req\mh neg\mh}a^i\} \cup \pre(a^i)_{|F}\\ &\cup \; \{ \neg\mathsf{active\mh}a^j : \neg a^j \in \pre(a^i)\},\\
\mr{cond}&=\{\emptyset\rhd\{\mathsf{busy\mh agent}(i)\mathsf{, \neg free\mh agent}(i)\mathsf{, active\mh}a^i\}\}\\
&\cup\;\{\emptyset\rhd\{\mathsf{req\mh neg\mh}a^j : \neg a^j \in \pre(a^i)\}\}.
\end{align*}
To select $a^i$, $\mathsf{req\mh neg\mh}a^i$ has to be false. For each negative concurrency constraint $\neg a^j$ of $a^i$, action $\mathsf{select\mh}a^i$ adds fluent $\mathsf{req\mh neg\mh}a^j$, preventing $a^j$ from being selected after $a^i$.

With this extension, we only need to check {\em positive} concurrency constraints in the application phase. We also redefine $\mathsf{end\mh}a^i$ such that fluents of type $\mathsf{req\mh neg\mh}a^i$ are reset in the cleanup phase, using the opposite effect of $\mathsf{select\mh}a^i$. The initial state and goal condition do not change since the new fluents are always false while no joint action is simulated. 

The second extension is to impose a bound $C$ on the number of atomic actions in the selection phase, resulting in a classical planning problem $\Pi_C'=\langle F_C',A_C',I_C',G_C' \rangle$. The fluent set $F'_C\supseteq F'$ extends $F'$ with fluents $\mathsf{count}(j)$, $0\leq j\leq C$. We add counter parameters to select and reset actions so that they can respectively increment and decrement the value of the counter. Crucially, no select action is applicable when $j=C$, preventing us from selecting more than $C$ actions. The benefit is to reduce the branching factor by restricting joint actions to have at most $C$ atomic actions.

We leave the following proposition without proof:

\begin{prop}
	The compilation $\Pi_C'$ that includes both proposed extensions is sound.
\end{prop}
Note that the compilation $\Pi_C'$ is not complete. For instance, consider a concurrent multiagent plan that contains a joint action involving 4 atomic actions. If $C < 4$, then the concurrent multiagent plan cannot be converted into an equivalent classical plan without exceeding the bound $C$.

\section{Experimental Results}
\label{sec:results}
We tested our compilations in four concurrent domains: \textsc{TableMover}, \textsc{Maze}, \textsc{Workshop} and \textsc{BoxPushing}\footnote{The code of the compilation and the domains are available at https://github.com/aig-upf/universal-pddl-parser-multiagent.}.

In each domain, we used three variants of our compilations: unbounded joint action size, and joint action size bounded by $C=2$ and $C=4$. In all variants, we used the extension that identifies negative concurrency constraints in the selection phase. The resulting classical planning problems were solved using Fast Downward~\cite{Helmert06} in the LAMA setting~\cite{richter:lama:JAIR2010}.
All experiments ran on Intel Xeon E5-2673 v4 @ 2.3GHz processors, with a time limit of 30 minutes and a memory limit of 8 GB.

The \textsc{Maze} domain \cite{crosby2014multiagent} consists of a grid of interconnected locations. Each agent in the maze must move from an initial location to a target location. The connection between two adjacent locations can be either a door, a bridge or a boat. A door can only be used by one agent at once, a bridge is destroyed upon first use, and a boat can only be used by two or more agents in the same direction.

The \textsc{Workshop} domain is a new domain in which the objective is to perform inventory in a high-security storage facility. It has the following characteristics:
\begin{itemize}
	\item To open a door, one agent has to press a switch while another agent simultaneously turns a key.
	\item To do inventory on a pallet, one agent has to use a forklift to lift the pallet while another agent examines it (for security reasons, labels are located underneath pallets).
	\item There are also actions for picking up a key, entering or exiting a forklift, moving an agent, and driving a forklift.
\end{itemize}

The \textsc{BoxPushing} domain \cite{BrafmanZoran14} consists in a grid of interconnected cells. Agents must push boxes from one cell to another. Boxes have different sizes and require different numbers of agents to push (1, 2 or 3).

We use two algorithms for comparison: \citeauthor{CrosbyJR14}~(\citeyear{CrosbyJR14}) and \citeauthor{ShekharB18}~(\citeyear{ShekharB18}), which we refer to as CJR and SB respectively. Both algorithms define concurrency constraints in the form of affordances on sets of objects. {For example, the affordance on the object set $\{\text{location}, \text{boat}\}$ in $\textsc{Maze}$ is defined as $\left[2, \infty\right]$ in CJR, representing that at least two agents have to row the boat between the same two locations at once. SB define the same affordance as $\left[2, 2\right]$, only allowing the minimum number of agents necessary to row a boat (i.e.~$2$).}

CJR do not separate atomic action selection from atomic action application. This is a problem since one of the atomic actions can delete the precondition of other atomic actions, thus canceling the formation of the joint action. For example, in the \textsc{Maze} domain, the action for crossing a bridge requires that the bridge exists, and destroys the bridge as an effect. Therefore, as this approach does not separate the selection from the application, this action can be done just by one agent at a time (and not by infinite agents as the problem states). The same occurs in the \textsc{BoxPushing} domain. Instances where a medium or a large box must be moved cannot be solved with this approach because the first agent to ``push'' the box will move it. Thus, the box location precondition for the other agent(s) does not hold, so the box is not moved in the end. On the other hand, SB extend CJR with mechanisms to avoid this problem, deferring effects until after all agents have applied their atomic actions.

Moreover, concurrency constraints in the form of object affordances are not as expressive as those of~\citeauthor{kovacs2012multi}~(\citeyear{kovacs2012multi}):
\begin{itemize}
	\item Actions cannot appear in conditional effects, making it impossible to model \textsc{TableMover} instances (SB present results from a simplified version without blocks).
	\item To define concurrency constraints, actions need at least one shared object, which is not the case in \textsc{Workshop}.
\end{itemize}

In experiments, we used Fast Downward in the LAMA setting to solve the instances produced by CJR and SB.

\begin{table*}[]
\centering
\resizebox{\textwidth}{!}{%

\begin{tabular}{lrrrrrrrrrrrrrrrrrrrrrrrrr}
\toprule
Domain              &$N$ & \multicolumn{5}{c}{Coverage}                                                                                                        & & \multicolumn{5}{c}{Time (s.)}                                                                                                    & & \multicolumn{5}{c}{Makespan}                                                                                                      & & \multicolumn{5}{c}{\# Grounded actions ($\times\num{e3}$)}           \\ \cmidrule{3-7} \cmidrule{9-13} \cmidrule{15-19} \cmidrule{21-25}
                    &    & \multicolumn{1}{c}{2} & \multicolumn{1}{c}{4} & \multicolumn{1}{c}{$\infty$} & \multicolumn{1}{c}{CJR }  & \multicolumn{1}{c}{SB}   & & \multicolumn{1}{c}{2} & \multicolumn{1}{c}{4} & \multicolumn{1}{c}{$\infty$} & \multicolumn{1}{c}{CJR } & \multicolumn{1}{c}{SB} & & \multicolumn{1}{c}{2} & \multicolumn{1}{c}{4} & \multicolumn{1}{c}{$\infty$} & \multicolumn{1}{c}{CJR } & \multicolumn{1}{c}{SB}  & & \multicolumn{1}{c}{2} & \multicolumn{1}{c}{4} & \multicolumn{1}{c}{$\infty$} & \multicolumn{1}{c}{CJR } & \multicolumn{1}{c}{SB} \\
\midrule
\textsc{Maze}       & 20 & \textbf{13}           & 8                     & 6                            & 11                        & 9                        & & 361.5                 & 444.2                 & \textbf{145.6}               & 195.1                    & 216.1                  & & 47.2                  & 22.0                  & \textbf{11.7}                & 77.3                     & 67.7                    & & 41.7                  & 69.3                  & \textbf{27.9}                & 156.8                    & 108.2                       \\
$a=10$              & 10 & \textbf{8}            & 6                     & 5                            & 7                         & 6                        & & 250.2                 & 575.6                 & \textbf{170.4}               & 228.4                    & 323.1                  & & 48.3                  & 25.0                  & \textbf{12.2}                & 79.6                     & 69.8                    & & 39.9                  & 67.4                  & \textbf{26.1}                & 119.3                    & 102.1                       \\
$a=15$              & 10 & \textbf{5}            & 2                     & 1                            & 4                         & 3                        & & \textbf{539.5}        & -                     & -                            & -                        & -                      & & \textbf{45.4}         & -                     & -                            & -                        & -                       & & 43.9                  & 71.8                  & \textbf{30.0}                & 194.3                    & 115.1                      \\
\midrule
\textsc{BoxPushing} & 20 & 9                     & 15                    & 16                           & -                         & \textbf{18}              & & \textbf{5.2}          & 36.4                  & 143.3                        & -                        & 305.8                  & & \textbf{11.2}         & 11.3                  & 12.9                         & -                        & 20.5                    & & 3.5                   & 5.7                   & 2.5                          & -                        & \textbf{2.0}                               \\
$a=2$               & 10 & 9                     & 9                     & 9                            & -                         & \textbf{10}              & & \textbf{5.2}          & 7.6                   & 6.0                          & -                        & 158.9                  & & \textbf{11.2}         & 11.9                  & 11.3                         & -                        & 18.4                    & & 1.8                   & 3.2                   & \textbf{1.1}                 & -                        & 1.2                               \\
$a=4$               & 10 & 0                     & 6                     & 7                            & -                         & \textbf{8}               & & -                     & \textbf{79.7}         & 319.9                        & -                        & 489.5                  & & -                     & \textbf{10.5}         & 15                           & -                        & 23.1                    & & 5.2                   & 8.2                   & 3.8                          & -                        & \textbf{2.9}                               \\
\midrule
\textsc{TableMover} & 24 & \textbf{15}           & 12                    & \textbf{15}                  & -                         & -                        & & \textbf{263.4}        & 336.7                 & 341.1                        & -                        & -                      & & \textbf{58.7}         & 59.0                  & 61.5                         & -                        & -                       & & 7.4                   & 13.1                  & \textbf{4.6}                 & -                        & -                           \\
$a=2$               & 12 & 10                    & 10                    & \textbf{11}                  & -                         & -                        & & \textbf{103.9}        & 226.6                 & 214.7                        & -                        & -                      & & 63.5                  & \textbf{62.0}         & 64.5                         & -                        & -                       & & 3.4                   & 6.1                   & \textbf{2.0}                 & -                        & -                           \\
$a=4$               & 12 & \textbf{5}            & 2                     & 4                            & -                         & -                        & & \textbf{582.4}        & -                     & -                            & -                        & -                      & & \textbf{49.0}         & -                     & -                            & -                        & -                       & & 11.5                  & 20.1                  & \textbf{7.2}                 & -                        & -                           \\
\midrule
\textsc{Workshop}   & 20 & \textbf{15}           & 13                    & 13                           & -                         & -                        & & 134.3                 & 301.4                 & \textbf{52.5}                & -                        & -                      & & 35.7                  & 37.0                  & \textbf{32.5}                & -                        & -                       & & 18.0                  & 31.0                  & \textbf{11.5}                & -                        & -                              \\
$a=4$               & 10 & \textbf{8}            & \textbf{8}            & \textbf{8}                   & -                         & -                        & & 42.8                  & 263.3                 & \textbf{37.1}                & -                        & -                      & & \textbf{37.3}         & 43.9                  & \textbf{37.3}                & -                        & -                       & & 7.7                   & 13.6                  & \textbf{4.8}                 & -                        & -                              \\
$a=8$               & 10 & \textbf{7}            & 5                     & 5                            & -                         & -                        & & 238.8                 & 362.3                 & \textbf{77.1}                & -                        & -                      & & 33.9                  & 26.0                  & \textbf{24.8}                & -                        & -                       & & 28.2                  & 48.3                  & \textbf{18.1}                & -                        & -                              \\
\bottomrule
\end{tabular}%
}
\caption{Summary of results; see text for details. $a$ is the number of agents, $N$ is number of instances; time and length are averages for all planners that solved at least 5 instances. The number of grounded actions is an average over all instances.}
\label{tab:req_conc_results}
\end{table*}

Table~\ref{tab:req_conc_results} shows the results for the four domains. To provide an idea of how each planner behaves as a function of the number of agents, the table shows for each domain the same metrics for different numbers of agents.

In terms of coverage (i.e. number of solved instances), the compilation variant bounded to 2 performs the best (52, 61.9\%). The unbounded compilation ($\infty$) and the variant bounded to 4 have similar coverage: 50 (59.5\%) and 48 (57.1\%) respectively. The performance of the variant bounded to 2 is not very good in \textsc{BoxPushing} for instances involving four agents because all of them require a large box to be pushed (i.e. three agents are required). Finally, SB and CJR are the approaches with the worst coverage: 27 (32.1\%) and 11 (13.1\%) respectively. The main reason is that they cannot solve \textsc{TableMover} and \textsc{Workshop} instances; CJR cannot solve \textsc{BoxPushing} instances either.

Regarding execution time, the unbounded compilation and the compilation bounded to 2 are the fastest. The higher the number of agents, the longer it takes to compute a plan.

In terms of makespan (i.e.~number of joint actions), our approach obtains the shortest plans. CJR and SB obtain longer plans because they only construct joint actions associated with specific concurrency constraints. Any atomic action that can be applied on its own thus becomes a joint action of size 1. In contrast, our approach can combine atomic actions arbitrarily and compress the solution while planning.

The main reason that SB works better in \textsc{BoxPushing} is due to the hardcoded representation of collaborative actions that involve a minimum number of agents. For example, to push a box that requires $b$ agents to move, SB defines collaborative actions that involve exactly $b$ agents, while in our case, a joint action involving more than $b$ agents will also satisfy the concurrency constraints. This results in a larger branching factor which in turn affects the performance.

Note however that such a minimalist representation of collaborative actions is not always complete. For example, we can define a \textsc{Maze} instance where three agents have to use a boat to cross a stream. If we only define collaborative actions that involve the minimum number of agents needed to row a boat (i.e.~2), such an instance becomes unsolvable since no sequence of 2 agents rowing the boat in different directions is capable of moving all three agents to the other side. In contrast, our approach can generate a joint action that allows all three agents to cross the stream concurrently.

We also performed a scalability experiments in the \textsc{Maze} domain. We compare our approach (the $\infty$ variant) to the naive approach of converting a MAP into a classical problem by creating a classical action for each combination of agents. The instances consisted of (1) a 3x3 grid, (2) a set of agents with the same initial and goal locations, and (3) a single path to the goal that consists of interleaved boats and bridges.

Table \ref{tab:scalability_maze} shows the number of grounded actions and solution time for varying numbers of agents. The naive approach cannot solve instances with 7 or more agents due to grounding, while our approach can solve instances with 100 agents.

\begin{table}[]
	\centering
	\resizebox{0.9\columnwidth}{!}{%
		\begin{tabular}{rrrrrr}
			\toprule
			\#Agents             & \multicolumn{2}{c}{\# Grounded actions}                 & & \multicolumn{2}{c}{Time (s.)}                            \\ \cmidrule{2-3} \cmidrule{5-6} 
			\multicolumn{1}{c}{} & \multicolumn{1}{c}{Naive} & \multicolumn{1}{c}{$\infty$}& & \multicolumn{1}{c}{Naive} & \multicolumn{1}{c}{$\infty$} \\ \midrule
			2                    & 48                        & 100                         & & 0.089                     & 0.226                        \\
			4                    & 992                       & 260                         & & 0.494                     & 0.226                        \\
			6                    & 31248                     & 484                         & & 53.864                    & 0.354                        \\
			8                    & -                         & 772                         & & -                         & 0.535                        \\
			10                   & -                         & 1124                        & & -                         & 0.758                        \\
			50                   & -                         & 21604                       & & -                         & 41.979                       \\
			100                  & -                         & 83204                       & & -                         & 289.887                      \\ \bottomrule
		\end{tabular}%
	}
	\caption{Scalability of our approach ($\infty$) compared to the naive compilation in the \textsc{Maze} domain.}
	\label{tab:scalability_maze}
\end{table}

\section{Related Work}

Several other authors consider the problem of concurrent multiagent planning. \citeauthor{BoutilierB01}~(\citeyear{BoutilierB01}) describe a partial-order planning algorithm for solving MAPs with concurrent actions, based on their formulation of concurrency constraints, but do not present any experimental results. CMAP~\cite{Borrajo13} produces an initial sequential plan for solving a MAP, but performs a post-processing step to compress the sequential plan into a concurrent plan.

\citeauthor{JonssonR11}~(\citeyear{JonssonR11}) present a best-response approach for MAPs with concurrent actions, where each agent attempts to improve its own part of a concurrent plan while the actions of all other agents are fixed. However, their approach only serves to improve an existing concurrent plan, and is unable to compute an initial concurrent plan. FMAP~\cite{TorrenoOS14} is a partial-order planner that also allows agents to execute actions in parallel, but the authors do not present experimental results for MAP domains that require concurrency.

The planner of~\citeauthor{CrosbyJR14}~(\citeyear{CrosbyJR14}) is similar to ours in that it also converts MAPs into classical planning problems. The authors only present results from the \textsc{Maze} domain, and concurrency constraints are defined as affordances on object sets that appear as arguments of actions. As we have seen, these concurrency constraints are not as flexible as those of~\citeauthor{BoutilierB01}~(\citeyear{BoutilierB01}).

\citeauthor{BrafmanZoran14}~(\citeyear{BrafmanZoran14}) extended the MAFS multiagent distributed algorithm \cite{NissimB14} to support actions requiring concurrency while preserving privacy. Messages are exchanged between agents in order to inform each other about the expansion of relevant states. Consequently, agents explore the search space together while preserving privacy. As pointed out by \citeauthor{ShekharB18}~(\citeyear{ShekharB18}), it has two main problems: (1) it does not consider the issue of subsumed actions, and (2) it does not support concurrent actions that affect each others preconditions. 

\citeauthor{MaliahBS17}~(\citeyear{MaliahBS17}) proposed MAFBS, which extended MAFS to use forward and backward messages. This approach reduced the number of required messages and resulted in an increase in the privacy of agents.

\citeauthor{ChouhanN16}~(\citeyear{ChouhanN16}, \citeyear{ChouhanN17}) proposed a PDDL-like language for specifying problems involving required concurrency, which is very similar to the one by \citeauthor{BoutilierB01}~(\citeyear{BoutilierB01}). Their planner does not make assumptions on the number of agents required to perform a joint action; rather, the number of agents is determined from the \emph{capability} of agents and the objects they are interacting with. For example, in a robot domain, the number of robots required to lift an specific object can depend on the weight of the object.

\citeauthor{ShekharB18}~(\citeyear{ShekharB18}) extended the planner of \citeauthor{CrosbyJR14}~(\citeyear{CrosbyJR14}). Thus, it is also based on compiling the multiagent problem into a classical problem. With respect to previous work, they added collaborative actions and removed all collaborative actions that are subsumed by others (i.e.~that do not involve a minimum number of agents). Besides, they showed that their approach can also be used in a distributed privacy preserving planner.

Compilations from multiagent to classical planning have also been considered by other authors. \citeauthor{muise-codmap15}~(\citeyear{muise-codmap15}) proposed a transformation to respect privacy among agents. The resulting classical planning problem was then solved using a centralized classical planner as in our approach. Besides, compilations to classical planning have also been used in temporal planning, obtaining state-of-the-art results in many of the International Planning Competition domains~\cite{conf/icaps/Jimenez15}.

\section{Conclusion}
This paper proposes a new compilation for concurrent multiagent planning problems. As far as we know, our algorithm is the first to handle concurrent conditional effects. In experiments we show that our approach is competitive with previous work, and that it can solve concurrent multiagent planning problems that are out of reach of previous approaches.

Since the number of atomic actions is exponentially smaller than the number of joint actions, a distributed action definition has the potential to scale to much larger instances, which we demonstrate in our experiments. It is not always easy to determine beforehand how many joint actions are needed; in \textsc{Maze}, we may need $k$ agents to cross a bridge together, requiring joint actions for $2,3,\ldots,k$ agents.

In future work, we would like to explore strategies for optimizing the makespan, improving scalability and introducing the notion of {\em capability} \cite{ChouhanN17}. We also want to automatically derive the bounds of our algorithm. Furthermore, privacy preserving is a central topic on multiagent planning; thus, this approach could be combined with suitable privacy-preserving mechanisms in the future.

\section{Acknowledgments}
This work has been supported by the Maria de Maeztu Units
of Excellence Programme (MDM-2015-0502). Anders Jonsson is partially supported by the grants TIN2015-67959 and PCIN-2017-082 of the Spanish Ministry of Science.


\begin{thebibliography}{}
	
	\bibitem[\protect\citeauthoryear{Borrajo}{2013}]{Borrajo13}
	Borrajo, D.
	\newblock 2013.
	\newblock {Plan Sharing for Multi-Agent Planning}.
	\newblock In {\em DMAP 2013 - Proceedings of the Distributed and Multi-Agent
		Planning Workshop at ICAPS},  57--65.
	
	\bibitem[\protect\citeauthoryear{Boutilier and Brafman}{2001}]{BoutilierB01}
	Boutilier, C., and Brafman, R.~I.
	\newblock 2001.
	\newblock {Partial-Order Planning with Concurrent Interacting Actions}.
	\newblock {\em J. Artif. Intell. Res. {(JAIR)}} 14:105--136.
	
	\bibitem[\protect\citeauthoryear{Brafman and Domshlak}{2008}]{BrafmanD08}
	Brafman, R.~I., and Domshlak, C.
	\newblock 2008.
	\newblock {From One to Many: Planning for Loosely Coupled Multi-Agent Systems}.
	\newblock In {\em Proceedings of the Eighteenth International Conference on
		Automated Planning and Scheduling, {ICAPS} 2008, Sydney, Australia, September
		14-18, 2008},  28--35.
	
	\bibitem[\protect\citeauthoryear{Brafman and Zoran}{2014}]{BrafmanZoran14}
	Brafman, R.~I., and Zoran, U.
	\newblock 2014.
	\newblock {Distributed Heuristic Forward Search with Interacting Actions}.
	\newblock In {\em Proceedings of the 2nd ICAPS Distributed and Multi-Agent
		Planning workshop (ICAPS DMAP-2014)}.
	
	\bibitem[\protect\citeauthoryear{Chouhan and Niyogi}{2016}]{ChouhanN16}
	Chouhan, S.~S., and Niyogi, R.
	\newblock 2016.
	\newblock {Multi-agent Planning with Collaborative Actions}.
	\newblock In {\em {AI} 2016: Advances in Artificial Intelligence - 29th
		Australasian Joint Conference, Hobart, TAS, Australia, December 5-8, 2016,
		Proceedings},  609--620.
	
	\bibitem[\protect\citeauthoryear{Chouhan and Niyogi}{2017}]{ChouhanN17}
	Chouhan, S.~S., and Niyogi, R.
	\newblock 2017.
	\newblock {{MAPJA:} Multi-agent planning with joint actions}.
	\newblock {\em Appl. Intell.} 47(4):1044--1058.
	
	\bibitem[\protect\citeauthoryear{Crosby, Jonsson, and
		Rovatsos}{2014}]{CrosbyJR14}
	Crosby, M.; Jonsson, A.; and Rovatsos, M.
	\newblock 2014.
	\newblock {A Single-Agent Approach to Multiagent Planning}.
	\newblock In {\em {ECAI} 2014 - 21st European Conference on Artificial
		Intelligence, 18-22 August 2014, Prague, Czech Republic - Including
		Prestigious Applications of Intelligent Systems {(PAIS} 2014)},  237--242.
	
	\bibitem[\protect\citeauthoryear{Crosby}{2013}]{Cr13b}
	Crosby, M.
	\newblock 2013.
	\newblock {A Temporal Approach to Multiagent Planning with Concurrent Actions}.
	\newblock {\em PlanSIG}.
	
	\bibitem[\protect\citeauthoryear{Crosby}{2014}]{crosby2014multiagent}
	Crosby, M.
	\newblock 2014.
	\newblock {\em {Multiagent Classical Planning}}.
	\newblock Ph.D. Dissertation, University of Edinburgh.
	
	\bibitem[\protect\citeauthoryear{Fox and Long}{2003}]{Fox:PDDL21:JAIR2003}
	Fox, M., and Long, D.
	\newblock 2003.
	\newblock {PDDL2.1: An Extension to PDDL for Expressing Temporal Planning
		Domains}.
	\newblock {\em J. Artif. Int. Res.} 20(1):61--124.
	
	\bibitem[\protect\citeauthoryear{Helmert}{2006}]{Helmert06}
	Helmert, M.
	\newblock 2006.
	\newblock {The Fast Downward Planning System}.
	\newblock {\em J. Artif. Intell. Res. {(JAIR)}} 26:191--246.
	
	\bibitem[\protect\citeauthoryear{Jim{\'e}nez, Jonsson, and
		Palacios}{2015}]{conf/icaps/Jimenez15}
	Jim{\'e}nez, S.; Jonsson, A.; and Palacios, H.
	\newblock 2015.
	\newblock {Temporal Planning With Required Concurrency Using Classical
		Planning}.
	\newblock In {\em Proceedings of the 25th International Conference on Automated
		Planning and Scheduling (ICAPS'15)}.
	
	\bibitem[\protect\citeauthoryear{Jonsson and Rovatsos}{2011}]{JonssonR11}
	Jonsson, A., and Rovatsos, M.
	\newblock 2011.
	\newblock {Scaling Up Multiagent Planning: {A} Best-Response Approach}.
	\newblock In {\em Proceedings of the 21st International Conference on Automated
		Planning and Scheduling, {ICAPS} 2011, Freiburg, Germany June 11-16, 2011}.
	
	\bibitem[\protect\citeauthoryear{Kovacs}{2012}]{kovacs2012multi}
	Kovacs, D.~L.
	\newblock 2012.
	\newblock {A Multi-Agent Extension of PDDL3.1}.
	\newblock In {\em Proceedings of the 3rd Workshop on the International Planning
		Competition (IPC)},  19--27.
	
	\bibitem[\protect\citeauthoryear{Maliah, Brafman, and Shani}{2017}]{MaliahBS17}
	Maliah, S.; Brafman, R.~I.; and Shani, G.
	\newblock 2017.
	\newblock {Increased Privacy with Reduced Communication in Multi-Agent
		Planning}.
	\newblock In {\em Proceedings of the Twenty-Seventh International Conference on
		Automated Planning and Scheduling, {ICAPS} 2017, Pittsburgh, Pennsylvania,
		USA, June 18-23, 2017.},  209--217.
	
	\bibitem[\protect\citeauthoryear{Muise, Lipovetzky, and
		Ramirez}{2015}]{muise-codmap15}
	Muise, C.; Lipovetzky, N.; and Ramirez, M.
	\newblock 2015.
	\newblock {{MAP-LAPKT}: Omnipotent Multi-Agent Planning via Compilation to
		Classical Planning}.
	\newblock In {\em Competition of Distributed and Multiagent Planners}.
	
	\bibitem[\protect\citeauthoryear{Nissim and Brafman}{2014}]{NissimB14}
	Nissim, R., and Brafman, R.~I.
	\newblock 2014.
	\newblock {Distributed Heuristic Forward Search for Multi-agent Planning}.
	\newblock {\em J. Artif. Intell. Res.} 51:293--332.
	
	\bibitem[\protect\citeauthoryear{Richter and
		Westphal}{2010}]{richter:lama:JAIR2010}
	Richter, S., and Westphal, M.
	\newblock 2010.
	\newblock The {LAMA} {P}lanner: {G}uiding {C}ost-{B}ased {A}nytime {P}lanning
	with {L}andmarks.
	\newblock {\em Journal of Artificial Intelligence Research} 39:127--177.
	
	\bibitem[\protect\citeauthoryear{Shekhar and Brafman}{2018}]{ShekharB18}
	Shekhar, S., and Brafman, R.~I.
	\newblock 2018.
	\newblock {Representing and Planning with Interacting Actions and Privacy}.
	\newblock In {\em Proceedings of the Twenty-Eighth International Conference on
		Automated Planning and Scheduling, {ICAPS} 2018, Delft, The Netherlands, June
		24-29, 2018.},  232--240.
	
	\bibitem[\protect\citeauthoryear{Stolba, Fiser, and Komenda}{2016}]{StolbaFK16}
	Stolba, M.; Fiser, D.; and Komenda, A.
	\newblock 2016.
	\newblock {Potential Heuristics for Multi-Agent Planning}.
	\newblock In {\em Proceedings of the Twenty-Sixth International Conference on
		Automated Planning and Scheduling, {ICAPS} 2016, London, UK, June 12-17,
		2016.},  308--316.
	
	\bibitem[\protect\citeauthoryear{Torre{\~{n}}o, Onaindia, and
		Sapena}{2014}]{TorrenoOS14}
	Torre{\~{n}}o, A.; Onaindia, E.; and Sapena, O.
	\newblock 2014.
	\newblock {{FMAP:} Distributed cooperative multi-agent planning}.
	\newblock {\em Appl. Intell.} 41(2):606--626.
	
\end{thebibliography}


\end{document}